\documentclass{article}
\usepackage[utf8]{inputenc}

\usepackage{geometry}
\newgeometry{vmargin={1in}, hmargin={1.25in,1.25in}} 
\usepackage{microtype,soul}
\usepackage{graphicx}
\usepackage{subfigure}
\usepackage{booktabs} % for professional tables
\usepackage{amsmath,amssymb,amsfonts,graphicx,nicefrac,mathtools,amsthm}
\usepackage{fancyhdr}
\usepackage{bbm}
\usepackage{tikz}
\usepackage{enumerate}
\usepackage{booktabs}
\usepackage{multirow}
\usepackage{xcolor}
\usepackage{enumitem}
\usepackage{makecell}
\usepackage{hyperref}

\def\R{\mathbb{R}}
\def\N{\mathbb{N}}
\def\G{\Gamma}

\def\F{\mathcal{F}}
\def\A{\mathcal{A}}

\def\I{\text{IG}}

\def\IG{\text{IG}}

\def\g{\gamma}

\def\pxi{\frac{\partial F}{\partial x_i}}
\def\pxj{\frac{\partial F}{\partial x_j}}

\def\b{\beta}

\def\a{\alpha}

\def\x{\bar{x}}

\theoremstyle{definition}

\newtheorem{theorem}{Theorem}

\newtheorem{lemma}{Lemma}
\newtheorem{corollary}{Corollary}

\newtheorem{definition}{Definition}

\title{Four Axiomatic Characterizations of the Integrated Gradients Attribution Method
}
\author{Daniel Lundstrom\thanks{lundstro@usc.edu, University of Southern California} 
\and Meisam Razaviyayn\thanks{razaviya@usc.edu, University of Southern California}}

\begin{document}
\maketitle

\begin{abstract}
Deep neural networks have produced significant progress among machine learning models in terms of accuracy and functionality, but their inner workings are still largely unknown. Attribution  methods seek to shine  a light on these ``black box" models by indicating how much each input contributed to a model's outputs. The Integrated Gradients (IG) method is a state of the art baseline attribution method in the axiomatic vein, meaning it is designed to conform to particular principles of attributions. We present four axiomatic characterizations of IG, establishing IG as the unique method to satisfy different sets of axioms among a class of attribution methods.\footnote{This work was supported in part with funding from the USC-Meta Center for Research and Education in AI \& Learning (REAL@USC center).}
\end{abstract}
% Deep learning has revolutionized 
% But the models are "black-box"
% explaining  would improve transparency and trust in  AI-powered decision making
% necessary for understanding other practical needs such as robustness and fairness

%%%%%%%%%%%%%%%%%%%%%%%%%%%%%%
%%%%%%  Introduction  %%%%%%%%
%%%%%%%%%%%%%%%%%%%%%%%%%%%%%%

\section{Introduction}
Deep neural networks have revolutionized various fields of machine learning over the past decade, from computer vision to natural language processing. These models are often left unexplained, causing practitioners difficulties when troubleshooting training inference issues or poor performance. This can lead to a lack of user trust in the model and an inability to understand what features are important to a model's function. Various regulations have been proposed that would require that ML models be transparent in certain scenarios \cite{WhitehouseAIBill}, \cite{EUAIBill}, \cite{UKAIBill}.

\vspace{2mm}
Attribution methods, sometimes called salience maps, are a response to this issue, purporting to explain the working of a model by indicating which inputs are important to a model's output. One group of such methods, game-theoretic attribution methods, go about producing attributions in a principled way by stipulating axioms, or guiding principals, and proposing methods that conform to those principles. When axioms are posited, the possible forms of an attribution become constrained, possibly to a single, unique method.

\vspace{2mm}
This is the case with the Integrated Gradients method. Initially introduced and analyzed in Axiomatic Attributions for Neural Networks~\cite{sundararajan2017axiomatic}, counterexamples to its uniqueness claims have since been provided by \cite{lundstrom2022rigorous} and \cite{lerma2021symmetry}.

\vspace{2mm}
While the original uniqueness claim about IG is problematic, in this work, we show that IG uniqueness claims can be established rigorously via different axioms. We start by introducing different axioms common to game-theoretic attribution methods, namely, implementation invariance, linearity, dummy, and completeness. Then, using axioms, we establish the following characterizations.

\begin{enumerate}
    \item Path methods can be characterized among attribution methods by the linearity, completeness, dummy, and non-decreasing positivity axioms.
    \item IG can be characterized among monotone path methods by the symmetry-preserving and affine scale invariance axioms.
    \item IG can be characterized among attribution methods by the linearity, affine scale invariance, completeness, non-decreasing positivity, and proportionality axioms.
    \item IG can be characterized among attribution methods by the linearity, completeness, dummy, and symmetric-monotonicity axioms.
    \item IG can be characterized among attribution methods by its action on monomials and the continuity of Taylor approximations for analytic functions axiom.
\end{enumerate}

\noindent Furthermore, we show that IG attributions to neural networks with ReLU and max functions coincide with IG attributions to softplus approximations to such models. This establishes a sort of continuity of IG among softplus approximations.

\section{Related Works}
Many solutions have been proposed to help explain black box neural networks. Using the taxonomy of \cite{linardatos2020explainable}, we can divide types of explainability methods into various overlapping categories. One approach is to make models intrinsically explainable \cite{letham2015interpretable}, while another method is explain models after the fact, called post-hoc explanations \cite{letham2015interpretable}. Some methods are designed to be used particular data type such as images \cite{smilkov2017smoothgrad} or language \cite{ventura2021explaining}. Some methods are designed for use on specific types of models \cite{vig2019multiscale}, while others are model agnostic \cite{ribeiro2016should}. Some methods explain the models workings over an entire data set \cite{ibrahim2019global}, while others explain the model's actions with regard to a particular input \cite{zeiler2014visualizing}.

\vspace{2mm}
Attributions methods are post-hoc methods designed to explain a model's action on a specific input. A particular approach to attributions is to apply methods from game-theoretic cost-sharing. These methods borrow from a developed set of literature, which provides them a strong theoretical background and established results. For example, the SHAP method \cite{lundberg2017unified} is an import of the Shapley value \cite{shapley1971assignment} into the ML attributions context. Likewise, the Integrated Gradient \cite{sundararajan2017axiomatic} is an import of the Aumann-Shapley cost-sharing method \cite{aumann1974values} into the ML attributions context.

\vspace{2mm}
Various works have analyzed the Aumann-Shapley method, characterizing it as the unique method to satisfy a set of desirable properties, or axioms. \cite{billera1982allocation} gave a characterization based on the idea of proportionality, while \cite{mirman1982demand} and \cite{samet1982determination} gave further characterization in a similar vein. \cite{mclean2004potential} showed a characterization based on the ideas of potential and consistency, while \cite{calvo2000value} characterized the Aumann-Shapley method based on the idea of balanced contributions. \cite{sprumont2005discrete} developed constraints around the merging or splitting agents to provide a characterization. \cite{young1985producer} provided a characterization using the principle of symmetric monotonicity, \cite{monderer1988values} developed another characterization absed on potential, while \cite{albizuri2014monotonicity} developed a characterization based on both merging/splitting and monotonicity.

\vspace{2mm}
% \daniel{do you want 2mm spaces everywhere like this?} 
The Integrated Gradients was first introduced in \cite{sundararajan2017axiomatic} and a characterization was provided for it as well. This claim did not cite any characterizations of the Aumann Shapley, but used the idea of preserving symmetry. However, \cite{lerma2021symmetry} and \cite{lundstrom2022rigorous} critiqued various aspects of the uniqueness claim with counterexamples and issues with the proof methods. The issues cited in \cite{lundstrom2022rigorous}'s criticism is that the ML context is significantly different than the cost-sharing context, causing unforeseen difficulties in applying results from one to another. Another characterization of IG was provide in \cite{sundararajan2020many}, this time based on a the cost-sharing result relying on the principle of proportionality. This proof was also criticized by \cite{lundstrom2022rigorous}, for the same reasons.

\section{Preliminaries}
In this section, we cover preliminaries needed for our work.

\subsection{Baseline Attribution Notations}
We begin by establishing preliminary notions. For $a$, $b\in \R^n$, let $[a,b]$ denote the hyper-rectangle with opposite vertices $a$ and $b$. Here $[a,b]$ represents the domain of input of a ML model, such as a colored image. We denote the set of ML models of interest $\F$, with $F\in \F$ being some function $F:[a,b]\rightarrow \R$, e.g. a deep learning model. Here we only consider one output of a model, so that if a model reports a probability vector of scores from a softmax layer, for instance, we only consider one entry of the probability vector.

\vspace{2mm}
Throughout the paper $x$ represents a general function input, $\x$  represents a particular input that is part of an attribution, and $x'$ denotes a reference baseline. A \textit{baseline attribution method} (BAM) explains a model by assigning scores to the components of an input indicating its contribution to the output $F(\x)$. We define a BAM as:

\begin{definition}[Baseline Attribution Method]
Given an input $\x\in [a,b]$ and baseline $x'\in [a,b]$, $F\in \F(a,b)$, a baseline attribution method is any function of the form $A:D_A\rightarrow\R^n$, where $D\subseteq [a,b]\times[a,b]\times \F$.
\end{definition}
A BAM reports a vector, so that $A_i(\x,x',F)$ reports the contribution of the $i^\text{th}$ component of $\x$ to the output $F(\x)$, given the reference baseline input $x'$. BAMs are a type of attribution with a baseline input used for comparison to the input $\x$, usually representing an absence of features. Often a baseline $x'$ is implicit for the model $F$, and we may drop writing $x'$ if it is unnecessary. It is not guaranteed that a BAM is defined for any input, as we will see in section~\ref{section:IG}. We denote the domain where an attribution is defined by $D_A$. 

\vspace{2mm}
There are two particular BAM's defined on different function classes we will discuss. Define $\F^1(a,b)$ to be the set of real analytic functions on $[a,b]$, and define $\A^1$ to be the set of BAMs defined on $[a,b]\times[a,b]\times \F^1(a,b)$. We may write $\F^1$ if $a$, $b$ is apparent.

\vspace{2mm}
The class of real analytic functions is well understood, but does not include many practical deep NNs, such as those which use the ReLU and max functions. To address these networks, define $\F^2(a,b)$, or $\F^2$ if $a$, $b$ is apparent, to be the set of feed-forward neural networks with a finite number of nodes on $[a,b]$ composed of real-analytic layers and ReLU layers. This includes fully connected, skip, residual, max, and softmax layers, as well as activation functions like sigmoid, mish, swish, softplus, and leaky ReLU.

\vspace{2mm}
Formally, let $n_0,...,n_m,m\in \N$, and for $1\leq k \leq m$, let $F^k:\R^{n_{k-1}}\rightarrow \R^{n_k}$ denote a real-analytic function. Let $S^k:\R^{n_k}\rightarrow \R^{n_k}$ to be any function of the form $S^k(x) = (f^k_1(x_1),...,f^k_{n_k}(x_n))$, where $f^k_i(x_k)$ is the identity mapping or the ReLU function. That is, $S^k$ performs one of either a pass through or a ReLU on each component, and could perform different operations on different components. Each function in $\F^2$ takes the form:

$$F(x) = S^m \circ F^m \circ S^{m-1} \circ F^{m-1} \circ ... \circ S^2 \circ F^2 \circ S^1 \circ F^1 (x),$$
where $\circ$ denotes function composition. Note that a multi-input max function can be formulated by a series of two-input max functions, and $\max(x,y) = \text{ReLU}(x-y) + y$. Thus neural networks with the max function can be reformulated using only the ReLU function, and $\F^2$ includes neural networks with the max function. Define $\A^2(D)$ (or $\A^2$) to be the set of BAMs defined on $D\subseteq [a,b]\times[a,b]\times (\F^1\cup \F^2)$.

\subsection{Axiomatic Approach}

The previous definition of a BAM is very broad, and includes many BAMs that do not track the importance of inputs. The axiomatic approach to attribution methods is to stipulate properties that can be imposed on $A$, limiting its structure and ensuring it accurately tracks feature contribution. It is even possible that a set of axioms constrains attribution methods to the degree that only one method satisfies all of them. In this case, the set of axioms would characterize the attribution method. We move to review axioms common to the literature.

\vspace{2mm}
Our first axiom, \textit{implementation invariance} \cite{sundararajan2017axiomatic}, can be stated as follows:
\begin{enumerate}[start]
    \item \textit{Implementation Invariance}: $A$ is not a function of model implementation, but solely a function of the mathematical mapping of the model's domain to the range.
\end{enumerate}
This axiom stipulates that an attribution method be independent of the model's implementation. Otherwise, the values of the attribution may carry information about implementation aspects such as architecture. This axiom requires that attributions ignore all aspects of specific implementation. Many methods, such as Smoothgrad \cite{smilkov2017smoothgrad} and SHAP \cite{lundberg2017unified}, satisfy implementation invarinace while \cite{sundararajan2017axiomatic} showed that DeepLIFT \cite{shrikumar2017learning} and Layer-Wise Relevance Propogation \cite{binder2016layer} do not satisfy it.

\vspace{2mm}
The next axiom, \textit{linearity} \cite{sundararajan2017axiomatic} \cite{sundararajan2020many} \cite{janizek2021explaining}, is given as,
\begin{enumerate}[resume*]\label{axiom:linearity}
    \item \textit{Linearity}: If $(\x,x',F)$, $(\x,x',G) \in D_A$, $\a,\b\in \R$, then $(\x,x',\a F+\b G)\in D_A$ and $A(\x,x',\a F+\b G) = \a A(\x,x',F)+\b A(\x,x',G)$.
\end{enumerate}
The linearity ensures that if $F$ is a linear combination of other models, a weighted average of model outputs for example, then the attributions of $F$ equals the average of the attributions to the sub-models. This imposes structure to the attributions outputs, so that if a model's outputs are scaled to give outputs twice as large for example, then the attributions are scaled as well.

\vspace{2mm}
We say that a function $F$ does not vary in an input $x_i$ if for every $x$ in the domain of $F$, $G(t):= F(x_1,...,x_{i-1},t,x_{i+1},...,x_m)$ is a constant function. We denote that $F$ does not vary in $x_i$ by writing $\partial_i F \equiv 0$. With is definition we may state another axiom, \textit{dummy}\footnote{The dummy axiom here is called Sensitivity(b) in \cite{sundararajan2017axiomatic}.},
\begin{enumerate}[resume*]\label{axiom:dummy}
    \item \textit{Dummy}: If $(\x,x',F)\in D_A$ and $\partial_i F \equiv 0$, then $A_i(\x,x',F) = 0$.
\end{enumerate}
Dummy ensures that whenever an input has no effect on the function, the attribution score is zero.

\vspace{2mm}
Another axiom, \textit{completeness} \cite{sundararajan2017axiomatic} \cite{sundararajan2020many} \cite{tsai2022faith}, is given as,
\begin{enumerate}[resume*]\label{axiom:completeness}
    \item \textit{Completeness}: If $(\x,x',F)\in D_A$, then $\sum_{i=1}^{n} A_i(\x,x',F) = F(\x)-F(x')$.
\end{enumerate}
Completeness grounds the meaning of the magnitude and sign of attributions. The magnitude of $A_i(\x,x',F)$ indicates that $\x_i$ contributed that quantity to the change in function value from $F(x')$ to $F(\x)$. The sign of $A_i(\x,x',F)$ indicates whether $\x_i$ contributed to function increase or function decrease. Thus the attributions to each input give a complete account of function change, $F(\x)-F(x')$.

\subsection{The Integrated Gradients}\label{section:IG}

There is a particular form of baseline attribution method which satisfied axioms~1-4, called a path method. Define a \text{path function} as follows:

\begin{definition}[Path Function]
A function $\g(\x,x',t):[a,b]\times[a,b]\times[0,1]\rightarrow [a,b]$ is a path function if, for fixed $\x,x'$, $\g(t):=\g(\x,x',t)$ is a continuous, piecewise smooth curve from $x'$ to $\x$.
% \vspace{-0.3cm}
\end{definition}

We may drop both $\x$, $x'$ when they are fixed, and write $\g(t)$. If we further suppose that $\pxi(\g(t))$ exists almost everywhere\footnote{$\pxi(\g(t))$ exists almost everywhere iff the Lebesgue measure of $\{t\in[0,1]:\pxi(\g(t))\text{ does not exist.}\}$ is $0$.}, then the \textit{path method} associated with $\g$ can be defined as:

\begin{definition}[Path Method]\label{def:path method} Given the path function~$\g(\cdot,\cdot,\cdot)$, the corresponding path method is defined as
\begin{equation}\label{eq:PathMethod}
    A^\g(\x,x',F) = \int_0^1 \frac{\partial F}{\partial x_i} (\g(\x,x',t)) \times \frac{\partial \g_i}{\partial t} (\x,x',t) dt,
\end{equation}
where $\g_i$ denotes the $i$-th entry of $\g$.
\end{definition}

Path methods are well defined when $\nabla F$ exists and is continuous on $[a,b]$, however, this is not necessarily the case for common ML models, such as neural networks that use ReLU and max functions. For example, if $\x=(1,1)$, $x'=(0,0)$, we use the straight line path $\g(t) = t(1,1)$. If $F(x)=\max(x_1,x_2)$, then $A^\g(\x,x',F)$ does not exist because the partial derivatives are undefined on any point on the path $\g$.

\vspace{2mm}
The \textit{Integrated Gradients method} \cite{sundararajan2017axiomatic} is the path method defined by the straight path from $x'$ to $\x$, given as $\g(\x,x',t) = x' + t(\x-x')$, and takes the form:
\begin{definition}[Integrated Gradients Method]%\label{eq:IG}
Given a function $F$ and baseline $x'$, the Integrated Gradients attribution of the $i$-th component of $\x$ is defined as
\begin{equation} \label{eq:IG} 
    \text{IG}_i(\x,x',F) = (\x_i-x_i')\int_0^1 \frac{\partial F}{\partial x_i} (x' + t(\x-x')) dt
\end{equation}
\end{definition}
The Integrated Gradients method is the application of the Aumann Shapley cost-sharing method applied to the ML attributions context~\cite{aumann1974values}. We define $D_\IG\subseteq [a,b]\times [a,b]\times (\F^1\cup \F^2)$ to be the domain where IG is defined.

\vspace{2mm}
The Integrated Gradient method satisfies the four axioms stated above. We give a brief explanation for how it satisfies each. Assume here that $\g(t)$ is the straight line IG path.

\begin{itemize}
    \item Implementation Invariance: IG only depends on $\nabla F$, which is independent on the implementation of $F$.
    \item Linearity: For any index $i$ we have:
        \begin{equation*}
            \begin{split}
                \IG_i(\x,x',aF+bG) &= (\x_i-x_i')\int_0^1 \frac{\partial(aF+bG)}{\partial x_i}(\g(t))dt\\
                &= a(\x_i-x_i')\int_0^1\frac{\partial F}{\partial x_i}(\g(t))dt + b(\x_i-x_i')\int_0^1\frac{\partial G}{\partial x_i}(\g(t))dt\\
                &= a\IG_i(\x,x',G) + b\IG_i(\x,x',G)
            \end{split}
        \end{equation*}
    \item Dummy: If $\partial_i F\equiv 0$ then $\IG_i(\x,x',F)$ integrates the zero function, and equals 0.
    \item Completeness: Letting `` $\cdot$ " denote the inner product, we employ the fundamental theorem of line integrals to gain:
    \begin{equation*}
            \begin{split}
                \sum_{i=1}^n\IG_i(\x,x',F) &= \sum_{i=1}^n(\x_i-x_i')\int_0^1 \frac{\partial F}{\partial x_i}(\g(t))dt\\
                &= \int_0^1 \nabla F(\g(t)) \cdot \g'(t)dt\\
                &= F(\x)-F(x')
            \end{split}
        \end{equation*}
\end{itemize}

\section{Path Method Characterization with NDP and Symmetry-Preserving}
The first attempt at characterizing the Integrated Gradients was presented in \cite{sundararajan2017axiomatic}. The general flow of the argument was 1) path methods uniquely satisfy a set of axioms, and 2) IG is the unique path method that satisfies certain extra properties. Later, \cite{lundstrom2022rigorous} critiqued the first part of the argument with provided counterexamples, while \cite{lerma2021symmetry} critiques and provided counterexamples to the second argument and adjusted some results. Here we present the current understanding of the argument and establish a characterization of IG among path methods by adding affine scale invariance, an axiom already seen in the literature.

\subsection{Characterizing Ensembles of Monotone Path Methods}
A path function $\g(t)$ from $x'$ to $\x$ is called monotone if each component is monotone in $t$. We denote the set of monotone paths from $x'$ to $\x$ by $\G^\lambda(\x,x')$. We say that $F$ is non-decreasing from $x'$ to $\x$ if $F(\g(t))$ is monotone in $t$ for each $\g\in\G^m(\x,x')$. We then define the axiom \textit{Non-Decreasing Positivity} (NDP) as follows:

\begin{enumerate}[resume*]\label{axiom:NDP}
    \item\textit{Non-Decreasing Positivity}: If $(\x,x',F)\in D_A$ and $F$ is non-decreasing from $x'$ to $\x$ then $A(\x,x',F)\geq0$.
\end{enumerate}

\noindent 
If $F$ is non-decreasing from $x'$ to $\x$, each input of $F$ does not cause a decrease if it moves closer to the input from the baseline. Thus, intuitively no component of $\x$ contributed to $F$ decreasing by being at its input value rather than the baseline value. Because no input contributed to $F$ decreasing in value, NDP asserts that those attributions should not be negative.

\vspace{2mm}
With NDP, we can characterize a sort of averaging of monotone path methods among all baseline attribution methods.\footnote{For an account of the differences between theorem~\ref{Theorem: ensembles of path methods in F^1} here and proposition~2 in \cite{sundararajan2017axiomatic}, see \cite{lundstrom2022rigorous}.}

\begin{theorem}\label{Theorem: ensembles of path methods in F^1}\cite[Theorem 2]{lundstrom2022rigorous}
Suppose $A\in\A^1$. The following are equivalent:

% \vspace{-0.4cm}

\begin{enumerate}
    \item[i.]  $A$ satisfies completeness, linearity, dummy, and NDP.
% \vspace{-0.3cm}
    \item[ii.] There exists a family of probability measures $\mu^{\cdot,\cdot}$ indexed on $(\x,x')\in[a,b]\times[a,b]$, where $\mu^{\x,x'}$ is a measure on $\G^m(\x,x')$, such that $$A(\x,x',F) = \int_{\G^m(\x,x')} A^\g(\x,x',F) d\mu^{\x,x'}(\g)$$
\end{enumerate}
\end{theorem}
% \vspace{-0.3cm}

\noindent 
Theorem~\ref{Theorem: ensembles of path methods in F^1} states that if $A\in\A^1$ is constrained according to the four axioms, then $A$ is an expected value of path methods with monotone paths. We call this expected value of path methods an \textit{ensemble of monotone path methods}, or more generally an ensemble of path methods if the expectation is not constrained to monotone paths. To present results for $\F^2$, we first give a result on the topology of NN models in $\F^2$:

\begin{lemma}\label{RS Lemma: path methods topology}\cite[Lemma 2]{lundstrom2022rigorous}
Suppose $F\in\F^2$. Then $[\x,x']$ can be partitioned into a nonempty region $U$ and its boundary $\partial U$, where $F$ is real-analytic on $U$, $U$ is open with respect to the (usual) topology of the dimension of $[\x,x']$, and $\partial U$ is measure $0$.
\end{lemma}

We now present a claim extending theorem~\ref{Theorem: ensembles of path methods in F^1} to functions in $\F^2$. Let $U$ denote the set as described above in Lemma~\ref{RS Lemma: path methods topology}, and denote the set of points on the path $\g$ by $P^\g$. 

\begin{theorem}\label{Theorem: ensembles of path methods in F^2}\cite[Theorem 3]{lundstrom2022rigorous}
Suppose $A\in\A^2$ is defined on $[a,b]\times[a,b]\times \F^1$ and some subset of $[a,b]\times[a,b]\times \F^2$, and satisfies completeness, linearity, dummy, and NDP. Let $\mu^{\cdot,\cdot}$ be the family of measures on monotone paths that defines $A$ on $[a,b]\times[a,b]\times \F^1$ from Theorem~\ref{Theorem: ensembles of path methods in F^1}, and let $(\x$, $x',F)\in[a,b]\times[a,b]\times \F^2$. If $A(\x,x',F)$ is defined, and for almost every path $\g\in\G^m(\x,x')$ (according to $\mu^{\x,x'}$), $\{t\in[0,1]: \g(t)\in \partial U\}$ is a null set w.r.t the Lebesgue measure on $\R$, then $A(\x,x',F)$ is equivalent an ensemble of monotone path methods. Furthermore, this ensemble is defined with the same $\mu^{\cdot,\cdot}$ as Theorem~\ref{Theorem: ensembles of path methods in F^1}.
\end{theorem}

The above result answers two questions: 1) is $A$ an ensemble of path methods when evaluating models in $\F^2$, and 2) is that ensemble the same ensemble that $A$ uses to evaluate models in $\F^1$? The above theorem guarantees that when considering models in $\F^2$ which may not be differentiable on $[a,b]$, $A$ is still an ensemble of path methods, and, in fact, is the same ensemble that define's $A$'s action on models in $\F^1$. Thus Theorem~\ref{Theorem: ensembles of path methods in F^2} establishes that while ensembles of path methods uniquely satisfy a set of axioms for attributions in $\A^1$, they also satisfy these axioms for models in $\F^2$, if it makes sense to do so.

\subsection{Characterizing IG Among Monotone Path Methods}
Among ensembles of monotone path methods, another popular method exists: the Shapley value \cite{shapley1971assignment}, \cite{lundberg2017unified}. The Shapley value is obtained by considering average change in function value when a component's value is changed from $x_i'$ to $\x_i$. Specifically, consider all possible ways that $x'$ can transition to $\x$ by sequentially toggling each component from $x_i'$ to $\x_i$. The Shapley value for $\x_i$ is the average change in function value over all possible transitions via toggling. This method can be formulated as an ensemble of $n!$ path methods. With speedups, calculating the Shapley value precisely is exponential in the number of inputs, and significant effort has been put into faster calculation via approximation \cite{chen2023algorithms}. The Shapley value was criticized as potentially problematic compared to IG in the original IG paper \cite{sundararajan2017axiomatic}[Remark 5].

\vspace{2mm}
As an alternative to this approach, the most computationally efficient ensemble would be an ensemble composed of a single-path method. It can be shown that IG is the unique path method that satisfies a couple of axioms.

\vspace{2mm}
The first axiom, \textit{symmetry-preserving}, is given as:
\begin{enumerate}[resume*]\label{axiom:symmetry preserving}
    \item\textit{Symmetry-Preserving}: For a vector $x$ and indices $1\leq i$, $j \leq n$, define $x^*$ by swapping the values of $x_i$ and $x_j$. Now suppose that $\forall x\in [a,b]$, $F(x) = F(x^*)$. Then if $(\x,x',F)\in D_A$, $\x_i=\x_j$ and $x'_i=x'_j$, we have $A_i(\x,x',F) = A_j(\x,x',F)$.
\end{enumerate}

\noindent Symmetry-preserving requires ``swappable" features with identical values to give identical attributions. This axiom was introduced in~\cite{sundararajan2017axiomatic}, but was criticized as insufficient to characterize IG among path methods in~\cite{lerma2021symmetry}. In short, other path methods exist that take the straight line path as IG does when $x_i'=x_j'$, $\x_i=\x_j$, but deviate otherwise. These counter-examples exist because symmetry-preserving only makes requirements when $x_i'=x_j'$, $\x_i=\x_j$, not otherwise. To remedy this, we considered strengthening the symmetry axioms, but found it insufficient to characterize IG among path methods. See Appendix~\ref{Appendix: symmetry-preserving commentary} for details.

\vspace{2mm}
The second axiom, \textit{Affine Scale Invaraince}, is given as:
\begin{enumerate}[resume*]\label{axiom:ASI}
    \item\textit{Affine Scale Invariance (ASI)}: For a given index $i$, constants $c\neq 0$, $d$, define the affine transformation $T(x) := (x_1,...,cx_i+d,...,x_n)$. Then whenever $\x, x', T(\x), T(x')\in [a,b]$, we have $A(\x,x',F) = A(T(\x),T(x'),F  \circ  T^{-1})$.
\end{enumerate}

\noindent This axiom can be justified by considering unit conversion. Suppose $F$ is some machine learning model where input $\x_i$ is given in degrees Fahrenheit. $T$ could be an affine transformation that converts the $i^\text{th}$ input from Fahrenheit to Celcius, so that $F\circ T^{-1}$ is an adjusted model where $\x_i$ would be given in Celsius, converted to Fahrenheit, then input into the original model. Affine scale invariance would require that an attribution method $A$ give the same attributions whether in Fahrenheit inputs, $(\x,x',F)$, or Celcius inputs, $(T(\x),T(x'),F\circ T^{-1})$.

\vspace{2mm}
It is interesting to note that ASI effectively means that the shape of a path for a path method stays the same regardless of the input or baseline values. Explicitly, suppose $A^\g$ is a path method satisfying ASI. For any $\x$, $x'$ there exists a unique affine transformation $T$ such that $T(x') = 0$, $T(\x) = 1$, where by 0 and 1 we mean the vectors with entries that are all zero or one, respectively. Thus $A^\g(\x,x',F) = A^\g(T(\x),T(x'),F\circ T^{-1}) = A^\g(1,0,F\circ T^{-1})$. The final expression uses the path $\g(1,0,t)$, and ignores the form of the path for $\g(\x,x',t)$. This causes the path to keep the same shape, so that all paths are an affine stretching of the base path from $x'=0$ to $\x=1$.

\vspace{2mm}
With symmetry-preserving and ASI, IG can be characterized among monotone path methods:

\begin{theorem}\label{Theorem: symmetry preserving and ASI} (Symmetry-Preserving Path Method Characterization on $\A^2$)
If $A\in\A^2(D_\IG)$ is a monotone path method satisfying ASI and symmetry-preserving, then it is the Integrated Gradients method. 
\end{theorem}

\noindent 
The proof of Theorem~\ref{Theorem: symmetry preserving and ASI} is relegated to Appendix~\ref{appendix: symmetry-preserving and ASI}.

\section{Characterizing IG with ASI and Proportionality}
A second attempt at characterizing the Integrated Gradients was presented in The Many Shapley Values for Model Explanation paper~\cite{sundararajan2017axiomatic}, which was also critiqued by \cite{lundstrom2022rigorous} later. Here we present the characterization.

\vspace{2mm}
The axiom of \textit{proportionality} states,
\begin{enumerate}[resume*]\label{axiom:proportionality}
    \item\textit{Proportionality}: If there exists $G:[a,b]\rightarrow \R$ such that for all $x\in [a,b]$, $F(x) = G(\sum_ix_i)$, then there exists $c\in\R$ such that $A_i(\x,0,F) = c\x_i$ for $1\leq i\leq n$.
\end{enumerate}
This axiom states that if $F$ can be expressed as a function of the cumulative quantity, $\sum_i x_i$, then each attribution is proportional to its contribution to $\sum_i \x_i$, namely $\x_i$. This axiom originates from the context of cost-sharing \cite{friedman1999three}, where each $\x_i$ may represent an investment. As an example, if the return on investment, $F(\x)$, is a function of the cumulative dollars invested, $\sum_i \x_i$, then proportionality asserts that the payout to each investor should be proportional to the amount invested. This principle does not always apply in cost-sharing problems, as when different investors make different kinds of contributions to an investment. This principle is fitting, however, when all investments are of the same kind so that the payout is simply a function of the total investment. Admittedly, this axiom appears at first glance to be more sensible in the cost-sharing context than the ML attributions context, and depends on the application of interest.

\vspace{2mm}
With proportionality and ASI, we can characterize IG:

\begin{theorem}\label{Theorem: proportionality characterization}(Proportionality Characterization on $\A^2$)
Suppose that $A\in \A^2(D_\IG)$. Then the following are equivalent:

% \vspace{-0.4cm}

\begin{enumerate}
    \item[i.] $A$ satisfies linearity, ASI, completeness, NDP, and proportionality.
% \vspace{-0.3cm}
    \item[ii.] $A$ is the Integrated Gradients method.
\end{enumerate}
\end{theorem}
% \vspace{-0.3cm}
%

\noindent 
The proof of Theorem~\ref{Theorem: proportionality characterization} is deferred to Appendix~\ref{appendix: proportionality characterization}. Note that unlike theorem~\ref{Theorem: symmetry preserving and ASI}, which characterized IG among monotone path methods, this is a much broader characterization, establishing that all BAMs in $\A^2$, only IG satisfies the given axioms.

\section{Characterizing IG with Symmetric Monotonicity}

We next present a characterization of IG employing the concept of \textit{monotonicity}. The axiom of monotonicity can be stated as:

\begin{enumerate}
    \item[8a.]\textit{Monotonicity}: Suppose $F\in\F^1$. Then,
    \begin{enumerate}
        \item[i.] If $\x_i\neq x_i'$, then $\frac{\partial F}{\partial x_i}(x) \leq \frac{\partial G}{\partial x_i}(x)$ $ \forall x \in [\x,x']$ implies $\frac{A_i(\x,x',F)}{\x_i-x_i'} \leq \frac{A_i(\x,x',G)}{\x_i-x_i'}$.
        \item[ii.] If $\x_i= x_i'$, then $A_i(\x,x',F) = 0$.
    \end{enumerate}
\end{enumerate}
To explain i., the term $\frac{A_i(\x,x',F)}{\x_i-x_i'}$ is the per-unit attribution of $\x_i$. We take the contribution of $\x_i$ to the change in $F$, denoted $A_i(\x,x',F)$, and divide it by the total change in $\x_i$. If $\x_i$ contributed to $F$ increasing, but $\x_i$ decreased from the baseline, then the per-unit attribution of $\x_i$ would be negative. As an example, suppose both derivatives are positive and $\x_i>x_i'$. If increasing $\x_i$ causes at least as great an increase for $G$ as it does for $F$, then according to monotonicity, the per-unit attribution of $\x_i$ should be at least as great for $G$ as $F$.

\vspace{2mm}
Requirement ii$.$ is the continuous extension of i$.$ to the $\x_i=x_i'$ case under the assumption of completeness and dummy. To demonstrate this extension, let $F\in \F^1$, so that $c\leq \frac{\partial F}{\partial x_i}\leq d$ on the bounded domain $[a,b]$. Then, by completeness and dummy, $A_i(\x,x',cx_i) = c(\x_i-x_i')$ and $A_i(\x,x',dx_i) = d(\x_i-x_i')$. Thus $c = \frac{A_i(\x,x',cx_i)}{\x_i-x_i'}\leq \frac{A_i(\x,x',F)}{\x_i-x_i'}\leq \frac{A_i(\x,x',dx_i)}{\x_i-x_i'} = d$. Now, as $\x_i\rightarrow x_i'$, we have $A_i(\x,x',F) \rightarrow 0$.

\vspace{2mm}
With the idea of monotonicity, one can assert a similar principle to the comparison between different inputs with the axioms, \textit{symmetric monotonicity}:
\begin{enumerate}\label{axiom:symmetric monotonicity}
    \item[8b.]\textit{Symmetric Monotonicity}: Suppose $A\in\A^1$, $F$, $G\in\F^1$. Then:
    \begin{enumerate}
        \item[i.] If $\x_i\neq x_i'$ and $\x_j\neq x_j'$, then $\frac{\partial F}{\partial x_i}(x) \leq \frac{\partial G}{\partial x_j}(x)$ $ \forall x \in [\x,x']$ implies $\frac{A_i(\x,x',F)}{\x_i-x_i'} \leq \frac{A_j(\x,x',G)}{\x_j-x_j'}$. 
        \item[ii.] If $\x_i= x_i'$, then $A_i(\x,x',F) = 0$.
    \end{enumerate}
\end{enumerate}
Symmetric monotonicity enforces that the principle of monotonicity can be applied between different inputs. With symmetric monotonicity, we give the following characterization of IG among methods in $\A^1$:

\begin{theorem}\label{Theorem: symmetric monotonicity for A1}(Symmetric Monotonicity Characterization on $\A^1$)
Suppose that $A\in \A^1$. Then the following are equivalent:

% \vspace{-0.4cm}

\begin{enumerate}
    \item[i.] $A$ satisfies completeness, dummy, linearity, and symmetric monotonicity.
% \vspace{-0.3cm}
    \item[ii.] $A$ is the Integrated Gradients method.
\end{enumerate}
\end{theorem}

\noindent 
The proof of Theorem~\ref{Theorem: symmetric monotonicity for A1} is located in Appendix~\ref{appendix: symmetric monotonicity for A1}.

\vspace{2mm}
To extend the results to $\A^2$, we consider two options. The first is to include NDP, and the second is to include a version of symmetric monotonicity that is formulated for functions that may not be differentiable. To do this, we replace the condition $\frac{\partial F}{\partial x_i}(x) \leq \frac{\partial G}{\partial x_j}(x)$ $ \forall x \in [\x,x']$ with a condition applicable to non-differentiable functions.

\vspace{2mm}
Supposing $F$, $G\in \F^2$, we define the statement $\frac{\partial F}{\partial x_i}(x) \leq \frac{\partial G}{\partial x_j}(x)$ \textit{locally approximately} to mean: $ \exists \epsilon > 0$ such that $ |z| < \epsilon$ implies $\frac{F(x_1,...,x_i+z,...,x_n) - F(x)}{z} \leq \frac{G(x_1,...,x_j+z,...,x_n) - G(x)}{z}$ whenever both terms exists. The above statement indicates we have something akin to $\frac{\partial F}{\partial x_i}(x) \leq \frac{\partial G}{\partial x_j}(x)$, using local secant approximations of the derivative. We now state \textit{$\mathcal{C}^0$-symmetric monotonicity}, an adjustment to symmetric monotonicity for BAMs in $\A^2$:

\begin{enumerate}\label{axiom:c0symmetricmonotonicity}
    \item[8c.]\textit{$\mathcal{C}^0$-Symmetric Monotonicity}: Suppose $A\in\A^2(D_\IG)$, $(\x,x',F)$, $(\x,x',G)\in D_\IG$. Then:
    \begin{enumerate}
        \item[i.] If $\x_i\neq x_i'$ and $\x_j\neq x_j'$, then $\frac{\partial F}{\partial x_i}(x) \leq \frac{\partial G}{\partial x_j}(x)$ locally approximately $ \forall x \in [\x,x']$ implies $\frac{A_i(\x,x',F)}{\x_i-x_i'} \leq \frac{A_j(\x,x',G)}{\x_j-x_j'}$. 
        \item[ii.] If $\x_i= x_i'$, then $A_i(\x,x',F) = 0$.
    \end{enumerate}
\end{enumerate}

\noindent We now extend the characterization of Theorem~\ref{Theorem: symmetric monotonicity for A1} for attributions in $\A^2$.

\begin{theorem}\label{Theorem: symmetric monotonicity for A2}(Symmetric Monotonicity Characterization on $\A^2$)
Suppose that $A\in \A^2(D_\IG)$. Then the following are equivalent:
\begin{enumerate}
    \item[i.] $A$ satisfies completeness, dummy, linearity, symmetric monotonicity, and NDP.
    \item[ii.] $A$ satisfies completeness, dummy, linearity, and $\mathcal{C}^0$-symmetric monotonicity.
    \item[iii.] $A$ is the Integrated Gradients method.
\end{enumerate}
\end{theorem}
\noindent The proof of Theorem~\ref{Theorem: symmetric monotonicity for A2} is located in Appendix~\ref{appendix: symmetric monotonicity for A2}.

\section{Characterization by the Attribution to Monomials}
Another means of characterizing attribution methods is to begin with a principle of attributing to simple functions.\footnote{This concept was explored in \cite{sundararajan2020shapley} for an interactions method, and later by \cite{lundstrom2023distributing} for IG gradient-based interactions methods. Here we state results from that paper for $\A^1$, and give a result on continuity into $\A^2$.} First, for $m\in\N_0^n$, define $[x]^m := x_1^{m_1}\cdot\cdot\cdot x_n^{m_n}$. Given a set baseline $x'$ and $m\in \N_0^n$, we employ a slight abuse of terminology and define a monomial to be any function of the form $F(x) = [x-x']^m$. Now, consider a simple example function we would like to perform attribution on $F(x_1,x_2) = (x_1-x_1')^{100}(x_2-x_2')$. The function $F$ evaluated at $\x=(x_1'+2,x_2'+2)$ yields $F(\x) = 2^{100}2^1=2^{101}$. Now, considering methods that satisfy completeness, the attribution question is: how to distribute $F(\x)-F(x')=2^{101}$ between $x_1$ and $x_2$.

\vspace{2mm}
One possibility is to consider $x_1$ and $x_2$ equal contributors, so that $A(\x,x',F) = (\frac{2^{101}}{2},\frac{2^{101}}{2})$. This is in fact, what the Shapley value attribution. For any monomial $F(x) = [x-x']^m$, the Shapley value gives attributions equally to each input such that $m_i \neq 0$. This seems a naive attribution, given the structure of $F$.

\vspace{2mm}
Another means of attributing to the inputs of $F$ is to consider the magnitude of $m_i$, the power of $\x_i$. Particularly, we could attribute to $\x_i$ proportionally to the number of times it is multiplied when evaluating $F(\x)$. An attribution following this guideline would yield: $A((x_1'+2,x_2'+2), x',(x_1-x_1')^{100}(x_2-x_2')) = (\frac{100}{101}2^{101},\frac{1}{101}2^{101})$, a result that appears equitable. In fact, this attribution coincides with the attribution of IG. For $m\in\N_0^n$ such that $m_i\neq 0$, we have

\begin{equation}\label{IG distribution of monomials}
    \begin{split}
        &\IG_i(\x,x',[x-x']^m) \\
        = &(\x_i-x_i') \int_0^1 \frac{\partial ([x-x']^m)}{\partial x_i} (x' + t(\x-x')) dt\\
        = & (\x_i-x_i')\int_0^1 m_i (t(\x_1-x_1'))^{m_1}\cdot\cdot\cdot (t(\x_i-x_i'))^{m_i-1}\cdot\cdot\cdot (t(\x_n-x_n'))^{m_n} dt\\
        = &(\x_i-x_i')\int_0^1 m_i t^{\|m\|_1-1} (\x_1-x_1')^{m_1}\cdot\cdot\cdot (\x_i-x_i')^{m_i-1}\cdot\cdot\cdot (\x_n-x_n')^{m_n} dt\\
        = &\frac{m_i}{\|m\|_1}[\x-x']^m
    \end{split}
\end{equation}

We may proceed from attributions on monomials to attributions on $\F^1$ by requiring a sort of continuity criteria. For $m\in\N_0^n$, define $[m]!:= m_1!\cdot\cdot\cdot m_n!$, and define $D^m F = \frac{\partial^{\|m\|_1} F}{\partial x_1^{m_1}\cdot \cdot \cdot \partial x_n^{m_n}}$. Recall that for $F\in\F^1$, the Taylor approximation of order $l$ centered at $x'$, denoted $F_l$, is given by:

\begin{equation}
    T_l(x) = \sum_{m\in\N_0^n, \|m\|_1\leq l} \frac{D^m(F)(x')}{[m]!}[x-x']^m
\end{equation}
The Taylor approximation for analytic functions has the property that $D^mT_l$ uniformly converges to $D^mF$ for any $m\in\N_0^n$ and $x\in[a,b]$. Thus, it seems natural to require that any attribution $A\in\A^1$ satisfy $\lim_{l \rightarrow \infty} A(\x,x',T_l) = A(\x,x',F)$. This is the principle behind the axiom \textit{Continuity of Taylor Approximation for Analytic Functions}, or what we may equivalently call the \textit{continuity condition}, given below:

\begin{enumerate}\label{axiom:continuity condition}
    \item[9] \textit{Continuity of Taylor Approximation for Analytic Functions}: If $A\in\A^1$, $(\x,x',F)\in[a,b]\times [a,b] \times \F^1$, then $\displaystyle{\lim_{l \rightarrow \infty}}A(\x,x',T_l) = A(\x,x',F)$, where $T_l$ is the $l^\text{th}$ order Taylor approximation of $F$ centered at $x'$.
\end{enumerate}

\noindent 
We now give the characterization of IG according to its actions on monomials:

\begin{theorem}\label{theorem: IG on monomials and continuity}
(Distribution of Monomials Characterization on $\A^1$)\cite[Corollary 3]{lundstrom2023distributing}
    Suppose $A\in \A^1$. Then the following are equivalent:
    \begin{enumerate}
    \item[i.]  $A$ satisfies continuity of Taylor approximation for analytic functions and acts on monomials as:
    $$A(\x,x',[x-x']^m) = \frac{m}{\|m\|_1} \times [\x-x']^m$$
    \item[ii.] $A$ is the Integrated Gradients method.
\end{enumerate}
\end{theorem}

We may proceed from $\F^1$ to $\F^2$ by considering a means of approximating a feed-forward neural network by an analytic function. Suppose $F\in\F^2$ is a feed-forward neural network with ReLU and Max functions. Note that the multi-input max function can be formulated as a series of dual input max functions, and the dual input max function can be formulated as $\max(a,b) = \text{ReLU}(a-b) + b$. Thus we may formulate $F$ using only the ReLU function. We may then define $F_\a$ to be the analytic approximation of $F$ given by replacing all instances of ReLU in $F$ with the parameterized softplus, $s_\a(z) = \frac{\ln(1+\exp(\a z))}{\a}$. We show in Appendix~\ref{appendix: softplus approximation uniformly converges} that this softplus approximation uniformly converges to the function $F$.

\vspace{2mm}
Before we give our result, we first give a technical theorem on the topology of $[a,b]$ with respect softplus approximations of functions in $\F^2$. Let $\nabla F$ denote the gradient of $F$, and let $\lambda$ denote the Lebesgue measure on $\R$. Then our result is as follows:

\begin{theorem}\label{theorem: softplus approximation topology}
For any $F\in\F^2$, there exists an open set $U\subseteq [a,b]$ such that $\lambda(U) = \lambda([a,b])$ and for each $x\in U$, the following hold:
\begin{itemize}
    \item There exists an open set containing $x$, $B_x$, and real analytic function on $[a,b]$, $H_x$, such that $F\equiv H_x$ on $B_x$.
    \item $\nabla F(x)$ exists.
    \item $\nabla F_\a(x)\rightarrow \nabla F(x)$ as $\a\rightarrow \infty$.
\end{itemize}
\end{theorem}
\noindent With this theorem, we give a result on IG's ability to uniquely extend into models in $\F^2$. 

\begin{corollary}\label{corollary: convergence of softplus approximaiton}
Let $(\x,x',F)\in \F^2$, and let $U$ be the set as in Theorem~\ref{theorem: softplus approximation topology}. Let $\g(t) = x' + t(\x - x')$ and suppose $\lambda(\{t\in [0,1]: \g(t)\in U\}) = 1$. Then:

$$\lim_{\a\rightarrow \infty} \I(\x,x',F_\a) = \I(\x,x',F)$$
\end{corollary}
\noindent Proofs of Theorem~\ref{theorem: softplus approximation topology} and Corollary~\ref{corollary: convergence of softplus approximaiton} are located in Appendices~\ref{appendix: softplus approximation topology} and~\ref{appendix: convergence of softplus approximaiton}, respectively.

\section{Conclusion}
We present a table of results, summarizing various characterizations of the Integrated Gradients method found in this paper.
\vspace{2mm}
\begin{table}[h!]
\begin{center}
\begin{tabular}{ |c|c|c|c|c|c|c| } 
 \hline
 Assumptions &  \thead{Theorems \\ \ref{Theorem: ensembles of path methods in F^2} $+$ \ref{Theorem: symmetry preserving and ASI}} & \thead{Theorem \\ \ref{Theorem: proportionality characterization}} & \thead{Theorem \\ \ref{Theorem: symmetric monotonicity for A2} i.} & \thead{Theorem \\ \ref{Theorem: symmetric monotonicity for A2} ii.} & \thead{Theorem \\ \ref{theorem: IG on monomials and continuity}}\\ 
 \hline\hline
 Linearity & x & x & x & x & x
 \\
  \hline
 Dummy& x & - & x & x & -
 \\
   \hline
 Completeness  & x & x & x & x & -
 \\
   \hline
 NDP& x & x & x & - & -
  \\
   \hline
 Path Method & x & - & - & - & -
 \\
   \hline
 Symmetry-Preserving & x & - & - & - & -
 \\
   \hline
 ASI & x & x & - & - & -
 \\
   \hline
 Proportionality & - & x & - & - & -
 \\
   \hline
 Symmetric Monotonicity & - & - & x & x & -
 \\
   \hline
 Distribution of Monomials & - & - & - & - & x
 \\
   \hline
 Continuity Condition & - & - & - & - & x
 \\   
   \hline

\end{tabular}
\caption{Each axiom and the section it is located in are listed under the ``Assumptions" column. Under each column with a ``Theorem" heading, the the set of axioms that characterized IG are marked. All results characterize IG among attributions in $\A^2$ except for Theorem~\ref{theorem: IG on monomials and continuity}, which characterized IG among attribution sin $\A^1$. Also, note that Theorem~\ref{Theorem: symmetric monotonicity for A2} i. assumes Symmetric monotonicity for functions in $\F^1$, while Theorem~\ref{Theorem: symmetric monotonicity for A2} ii. assumes $\mathcal{C}^0$-symmetric monotonicity.}
\end{center}
\end{table}

The axiomatic approach to attributions provides benefits, among which are: 1) identifying shortcomings in existing methods, 2) providing guiding principles for the development of attribution methods, and 3) identifying methods that uniquely satisfy a set of desirable properties. The presented characterizations of the IG extend our knowledge of baseline attribution methods and lead to a definite and singular method that satisfies certain desirable properties. The community should consider these characterizations and their merit.

\vspace{2mm}
Even with these characterizations, it is unlikely that there is a single best attribution method. In cost-sharing, it has been established that no one method can satisfy all desirable properties,\footnote{See \cite[Lemma 4]{friedman1999three}.} and it is possible that this is the case for attribution methods as well. Theorems~\ref{Theorem: ensembles of path methods in F^1} \& \ref{Theorem: ensembles of path methods in F^2} demonstrate that ensembles of path methods uniquely satisfy common and broadly used axioms. However, the community should consider the less common axioms that help characterize IG: Theorem~\ref{Theorem: symmetry preserving and ASI} - being a symmetry preserving single path method, Theorem~\ref{Theorem: proportionality characterization} - proportionality, Theorems~\ref{Theorem: symmetric monotonicity for A1} \& \ref{Theorem: symmetric monotonicity for A2} - symmetric monotonicity, and Theorem~\ref{theorem: IG on monomials and continuity} - IG's distribution of monomials. In the author's opinion, these more defining axioms indicate contexts where the IG attribution method is preferable. We expect that there are other properties, suitable in other contexts, which would exclude the IG and recommend another method.

\clearpage

%%%%%%% Bibliography %%%%%%%%
\clearpage 
\bibliography{ref_XAI}
\bibliographystyle{alpha}

\clearpage

\appendix
\section*{Appendix}

\section{Symmetry-Preserving Alone is Insufficient to Characterize IG Among Path Methods}\label{Appendix: symmetry-preserving commentary}
Here we provide a counterexample to the claim that IG is the unique path method that satisfies symmetry-preserving. We also give a axiom that is stronger than symmetry-preserving, and show that this axiom is insufficient to characterize IG.

\subsection{Another Path Method that Satisfies Symmetry-Preserving}

Let $D=[0,1]^2$ and define $\g'(\x,x',t)$ element-wise as follows:
\begin{equation*}
\g_i(\x,x',t) = x_i' + (\x_i-x_i')t^{(\x_i-x_i')^2}
\end{equation*}
Note that $\g$ is monotonic. When $\x_1=\x_2$, $x_1'=x_2'$, then $\x_1-x_1'=\x_2-x_2'$, and $\g_1(\x,x',t) = \g_2(\x,x',t)$. In this case $\g$ is the straight path, and $A^\g$ acts as IG. Thus, $A^\g$ satisfies symmetry preserving. However, when $\x_1-x_1'\neq \x_2-x_2'$, then the path $\g$ differs from the straight line path, causing $A^\g\neq \IG$.

\subsection{Strong Symmetry-Preserving}\label{appendix: strong symmetry}
Here we present an attempt to strengthen symmetry and show that multiple path methods satisfy the strengthened axiom. The axiom, \textit{Strong Symmetry-Preserving}, extends the symmetry-preserving axiom to cases when $\x_i\neq \x_j$, $x_i' \neq x_j'$:

\begin{enumerate}[resume*]
    \item (Strong Symmetry-Preserving): For a vector $x$ and indices $1\leq i$, $j\leq n$, let $x^*$ denote the vector $x$ but with the $i^\text{th}$ and $j^\text{th}$ components swapped. Suppose $F$ is symmetric in $i$ and $j$, meaning that $F(x) = F(x^*)$ for all $x\in[a,b]$. Then $A_i(\x,x',F) = A_j(\x^*,x'^*,F)$.
\end{enumerate}
Here we provide a counterexample to the claim that IG is the unique path method that satisfies strong symmetry-preserving.

Let $D=[0,2]^2$, and let $F$ be symmetric in components 1 and 2. We define a path function $\g(t)$ that is equivalent to the IG path except when $\x = (2,1)$, $x' = (1,0)$  or $\bar{y} = \x^* = (1,2)$, $y' = x'^* = (0,1)$. For $\x$, $x'$, let $\g(t)$ be the path that travels in straight lines along the course: $(1,0)\rightarrow (2,0) \rightarrow (2,1)$. Now for baseline $y' = (0,1)$ and input $\bar{y} = (1,2)$, let $\g(t)$ be the path that travels in straight lines along the course: $(0,1)\rightarrow (0,2) \rightarrow (1,2)$.

We then have $A_1^\g(\x,x',F) = F(2,0)-F(1,0) = F(0,2) - F(0,1) = A_2^\g(\x^*,x'^*,F)$, and likewise $A_2^\g(\x,x',F) = A^{\g}_1(\x^*,x'^*,F)$.  Thus we have another strong symmetry-preserving path method that is not the IG path.

\section{Proof of Theorem~\ref{Theorem: symmetry preserving and ASI}}\label{appendix: symmetry-preserving and ASI}

\begin{proof}
We present an adjusted version of the proof found in \cite[Theorem 1]{sundararajan2017axiomatic}. Suppose that $A$ is a monotone path method that satisfies symmetry preserving and affine scale invariance. Then for any $i$, $A_i = A_i^\g = \int_0^1 \pxi(\g(t))\frac{d \g}{dt}(t) dt$ for some monotone path function $\g$. Let 1 and 0 denote the vectors of all ones and all zeros, respectively. We proceed by contradiction and suppose that $\g(\x,x',t) \neq (\x-x')t + x'$ when $\x=1$, $x'=0$. Particularly, we suppose that $\g(1,0,t) \neq 1 \times t$ (using the $n$-dimensional ones vector). WLOG, suppose that there exists a $t$ such that $\g_1(1,0,t) > \g_2(1,0,t)$. Let $(t_a,t_b)$ be the maximal open set such that if $t\in (t_a,t_b)$ then $\g_1(1,0,t) > \g_2(1,0,t)$.

We now move to define a $(\x,x',F)\in D_\IG$  where $F$ is symmetric in $x_1$ and $x_2$, but $A$ does not give equal attributions to $\x_1$ and $\x_2$. Let $\x = 1$, and $x' = 0$, and define $F\in \F^2$ by

$$F(x) = \text{ReLU}\left [-\text{ReLU}(x_1x_2 - t_a^2)+t_b^2-t_a^2 \right ]$$

\noindent Now, $F$ can be written in a case-format as follows:

\begin{equation}
    F(x) = 
    \begin{cases}
        t_b^2-t_a^2 & \text{if } x_1x_2\leq t_a^2\\
        t_b^2 - x_1x_2 & \text{if } t_a^2\leq x_1x_2\leq t_b^2\\
        0 & \text{if } x_1x_2 \geq t_b^2
    \end{cases}
\end{equation}

\noindent It is easy to verify that $(\x,x',F)\in D_\IG$. Calculating $A_i^\g(1,0,F)$, and using the short hand $\g(t) = \g(1,0,t)$, we gain:
\begin{equation}
\begin{split}
    A_1^\g(1,0,F) &= \int_0^1 \frac{\partial F}{\partial x_1}(\g(t))\frac{d \g}{dt}(t) dt\\
    &= \int_{t_a}^{t_b} \frac{\partial (t_b^2-x_1x_2)}{\partial x_1}(\g(t)) \times 1 dt\\
    &= \int_{t_a}^{t_b} -\g_2(t) \, dt
\end{split}
\end{equation}
and
\begin{equation}
\begin{split}
    A_2^\g(1,0,F) = \int_{t_a}^{t_b} -\g_1(t) \, dt
\end{split}
\end{equation}
By assumption, $\g_1(t)>\g_2(t)$ for $t\in (a,b)$, yielding:

\begin{equation}
    \begin{split}
        A_1^\g(1,0,F) &= \int_{t_a}^{t_b} -\g_2(t) \, dt\\
        &> \int_{t_a}^{t_b} -\g_1(t) \, dt\\
        & = A_2^\g(1,0,F)
    \end{split}
\end{equation}

This is a contradiction. Thus, there is no $t$ where $\g_1(1,0,t) >  g_2(1,0,t)$, and more generally, there is not $t$ were $\g_i(1,0,t) >  g_j(1,0,t)$ for any $i$, $j$. So, $\g_i(1,0,t) = g_j(1,0,t)$ for any pair, and we have the IG path. Thus $\g(1,0,t) = 1 \times t$.
 \vspace{2mm}
 
Now, consider any $(\x,x',F)\in\A^2(D_\IG)$, and use the shorthand $\g(t) = \g(1,0,t)$. Let $T$ be the affine mapping such that $T(1) = \x$, $T(x') = 0$. We employ the assumption that $A$ satisfies affine scale invariance to gain:

\begin{equation}\label{ASI technique}
    \begin{split}
        A_i(\x,x',F) &= A_i(T(1),T(0),F)\\
        &= A_i(\x,x',T(F))\\
        &= \int_0^1 \frac{\partial T(F)}{\partial x_i}(\g(t))\frac{d \g_i}{d t} dt\\
        &= \int_0^1 \frac{\partial (F\circ T)}{\partial x_i}(1\times t)dt\\
        &= \int_0^1 \frac{\partial F}{\partial x_i}(T(1 \times t)) \frac{\partial (T)_i}{\partial x_i} (1\times t)\\
        &= \int_0^1 \frac{\partial F}{\partial x_i}(x' + t(\x-x')) (\x_i-x_i') dt\\
        &= (\x_i-x_i')\int_0^1 \frac{\partial F}{\partial x_i}(x' + t(\x-x')) dt\\
        &= \I_i(\x,x',F)
    \end{split}
\end{equation}

\end{proof}

\section{Proof of Theorem~\ref{Theorem: proportionality characterization}}\label{appendix: proportionality characterization}
We set out to establish the results for $\A^1$, then move to establish the results for $\A^2$.

\subsection{Proof for $A\in\A^1$}
\begin{proof}
Here we present a proof along the lines of that found in \cite{sundararajan2020many}. This proof was criticized in \cite{lundstrom2022rigorous} and partially rectified, but we present the argument in full.

Suppose $A\in\A^2$.

\vspace{2mm}

(ii. $\Rightarrow$ i) IG satisfies linearity and completeness and proportionality because it is a path method. Suppose $F$ is non-decreasing from $x'$ to $\x$, then $(\x_i-x_i')$ and $\pxi(x'(t(\x-x'))$ do not have opposite signs, so
\begin{equation*}
    \IG_i(\x,x',F) = (\x_i-x_i')\int_0^1\pxi(x'+t(\x-x'))dt \geq 0,
\end{equation*}
which shows IG satisfies NDP. Finally, let $F(x) = G(\sum_j x_j)$ and $x'=0$. Then $\pxi(\x) = G'(\sum_j \x_j)$, and 
\begin{equation*}
    \IG_i(\x,x',F) = \x_i\int_0^1G'(\sum_j\x_j)dt
\end{equation*}
Note that the integral is equivalent for any $i$, so we take $c = \int_0^1G'(\sum_jx_j)dt$ to gain $\IG_i(\x,x',F) = cx_i$. Thus IG satisfies proportionality, and ii. $\Rightarrow$ i.

\vspace{2mm}

(i. $\Rightarrow$ ii.) Now suppose that $A$ satisfies linearity, ASI, completeness, NDP, and proportionality. Let $\A^0$ denote the set of all BAMs such that 1) they are defined on analytic, non-decreasing functions, 2) they are only defined for $x'=0$, $\x\geq0$, 3) the BAMs give non-negative attributions and 4) the BAMs satisfy completeness. By \cite[Theorem 3]{friedman1999three}, the only BAM in $\A^0$ to satisfy proportionality and ASI is the Integrated Gradients method.

Let $A\in\A^1$, and note that if $x'=0$, $\x\geq0$, $F$ non-decreasing, then $A(1,0,F)\geq 0$ by NDP. $A$ also satisfies completeness by assumption. Thus if we let $A'$ denote $A$ with the requisite restriction of domains, then $A'\in\A^0$. Because $A'$ satisfies ASI and proportionality, $A'=\IG$ on this restricted domain.

Let $x'=0$, and $\x=1$, the vector of all ones. For any $F\in\F^1$, $F$ is Lipschitz on bounded domain, and there exists $c\in\R^n$ such that $c\geq 0$, $F(x) + c^\intercal x$ is non-decreasing. Thus

\begin{equation*}
    \begin{split}
        A(1,0,F(x)) &= A(1,0,F(x) + c^\intercal x - c^\intercal x )\\
        &= A(1,0,F(x) + c^\intercal x) - A(1,0, c^\intercal x)\\
        &= \IG(1,0,F(x) + c^\intercal x) - \IG(1,0, c^\intercal x)\\
        &= \IG(1,0,F(x))\\
    \end{split}
\end{equation*}
We can then harness ASI as in Eq.~\ref{ASI technique} to get that $A(\x,x',F) = \IG(\x,x',F)$ for any $\x$, $x'$.

\end{proof}

\subsection{Proof for $A\in\A^2(D_\IG)$}
\begin{proof}
Let $A\in\A^2(D_\IG)$. It is easy to show ii. $\Rightarrow$ i.. We turn to show i. $\Rightarrow$ ii..

Suppose $A$ satisfies linearity, ASI, completeness, NDP, and proportionality. Let $(\x,x',F)\in D_\IG$, and choose a component $i$. By methods found in the proof of \cite[Theorem 2]{lundstrom2022rigorous}, there exists a sequence of functions $F_m$ such that:

     \begin{itemize}
         \item $F_m$ is analytic for all $m$.
         \item $\frac{\partial F_m}{\partial x_i} \leq \frac{\partial F}{\partial x_i}$ where $\frac{\partial F}{\partial x_i}$ exists.
         \item $\lim_{m\rightarrow \infty} \frac{\partial F_m}{\partial x_i} = \frac{\partial F}{\partial x_i}$ where $\frac{\partial F}{\partial x_i}$ exists.
         \item $|\frac{\partial F_m}{\partial x_i}|\leq k$ for all $m$.
         \item $F-F_m$ is non-decreasing from $x'$ to $\x$ in $i$.
     \end{itemize}
     
    $F-F_m$ is Lipshitz because $F$, $F_m$ are Lipshitz. Thus, for each $m$, there exists $c\in\R^n$ such that $c_i=0$ and $F(x)-F_\lambda(x) + c^\intercal x$ is non-decreasing from $x'$ to $\x$. Since $c^\intercal x\in\F^1$, we apply previous results to gain $A_i(\x,x',c^\intercal x) = \IG_i(\x,x',c^\intercal x) = 0$. Thus,

    \begin{equation*}
        \begin{split}
            A_i(\x,x',F(x)) - A_i(\x,x',F_m(x)) &= A_i(\x,x',F(x)) - A_i(\x,x',F_m(x)) + A_i(\x,x',c^\intercal x)\\
            &= A_i(\x,x',F(x) - F_m(x) + c^\intercal x)\\
            &\geq 0
        \end{split}
    \end{equation*}
    Thus we have $A_i(\x,x',F)\geq A_i(\x,x',F_m)$.
    
    Now, because $(\x,x',F)\in D_\IG$, $\int_0^1\pxi(x'+t(\x-x'))dt$ exists and $\pxi$ exists almost everythere on the path $x'+t(\x-x')$. Employing DCT, we have:

    \begin{align*}
        A_i(\x,x',F) &\geq \lim_{m\rightarrow \infty} A_i(\x,x',F_m)\\
        &= \lim_{m\rightarrow \infty} \I_i(\x,x',F_m)\\
        &= \lim_{m\rightarrow \infty} \int_0^1 (\x_i-x_i')\frac{\partial F_m}{\partial x_i}(\g(t)) dt\\
        &= \int_0^1 (\x_i-x_i')\frac{\partial F}{\partial x_i}(\g(t)) dt\\
        &= \I_i(\x,x',F)
    \end{align*}
     We may also gain the reverse, $A_i(\x,x',F) \leq \I_i(\x,x',F)$, using a similar method. Thus $A_i(\x,x',F) = \I_i(\x,x',F)$, concluding the proof.
\end{proof}

\section{Proof of Theorem~\ref{Theorem: symmetric monotonicity for A1}}\label{appendix: symmetric monotonicity for A1}
\begin{proof}
    
ii. $\Rightarrow$ i.) Let $A\in\A^1$ be the IG method, and let $(\x,x',F)$, $(\x,x',G) \in [a,b]\times[a,b]\times \F^1$. If $\x_i = x_i'$, then it is easy to confirm that $\IG(\x,x',F) = 0$. Suppose $\x_i\neq x_i'$, $\x_j\neq x_j'$. Then, supposing $\pxi \leq \pxj$, we have:

\begin{equation*}
    \begin{split}
        \frac{\IG_i(\x,x',F)}{\x_i-x_i'} &= \int_0^1 \pxi(x'+t(\x-x'))dt\\
        &\leq \int_0^1 \pxj(x'+t(\x-x'))dt\\
        &= \frac{\IG_j(\x,x',F)}{\x_j-x_j'}
    \end{split}
\end{equation*}
and IG satisfies symmetric monotonicity.

\vspace{2mm}

i. $\Rightarrow$ ii.) The following proof is inspired by \cite[Theorem 1]{young1985producer}. We begin with an important lemma:

\begin{lemma}
    Let $A\in \A^1$ satisfy completeness, dummy, linearity, and symmetric monotonicity. Then $A(\x,x',[x-x']^m) = \I(\x,x',[x-x']^m)$, where $m\in \N_0^n$.
\end{lemma}

\begin{proof}
Let $A\in \A^1$ satisfy completeness, dummy, linearity, and symmetric monotonicity. Fix $\x$, $x'$. It is useful to note that $\I_i(\x,x',[x-x']^m) = \frac{m_i}{\|m\|_1}[\x-x']^m$. We proceed by lexicographic induction on $m\in\N_0^n$. What we mean by $m'<_{\text{lex}} m$ is that $m_i'=m_i$ for $1\leq i <k$, but $m_k'<m_k$.

Let $M\subseteq \N_0^n$ be the set of values of $m$ for which $A(\x,x',[x-x']^m) = \I(\x,x',[x-x']^m) = \frac{1}{\|m\|_1}(m_1,...,m_n)[\x-x']^m$. Now, $A(\x,x',[x-x']^0) = 0 = \I(\x,x',[x-x']^0)$ by dummy, so $(0,...,0) \in M$. Suppose instead that $\|m\|_0 = 1$, so that only $m_i \neq 0$. By dummy, $A_j(\x,x',[x-x']^m) = 0$ for $j\neq i$, and by completeness, $A_i(\x,x',[x-x']^m)=[\x-x']^m$. Thus $A(\x,x',[x-x']^m) = \I(\x,x',[x-x']^m)$, and $\|m\|_0=1$ implies $m \in M$.

Suppose there exists some element in $\N_0^n$ that is not an element in $M$. Let $m^*$ be the smallest such element. Define $S =\{1\leq i\leq n : A_i(\x,x',[x-x']^{m^*}) \neq \I_i(\x,x',[x-x']^{m^*})\}$. By the above, we have that $\|m^*\|_0 \geq 2$. Note that if $i\in S$ then it must be that 1) $\x_i\neq x_i'$, for otherwise $A_i = 0 = \IG_i$, and 2) $m_i^* > 0$.

Choose $i$ to be the least element in $S$. $A$ and $\IG$ must disagree in two or more components, for if they disagreed in exactly one component, then they could not both satisfy completeness. Thus $i<n$. Define $F(x) = [x-x']^{m^*}$ and define,

$$G(x) = \frac{m^*_i}{m_n^*+1} (x_1-x_1')^{m_1^*}\dotsb (x_i-x_i')^{m_i^*-1} \dotsb (x_n-x_n')^{m_n^*+1} $$

Note $\frac{\partial F}{\partial x_i} = \frac{\partial G}{\partial x_n}$. Thus, we have by symmetric monotonicity:

$$\frac{A_i(\x,x',F)}{\x_i-x_i'} = \frac{A_n(\x,x',G)}{\x_n-x_n'}$$

Also note that $m^{**} = (m_1,...,m_i-1,...,m_n+1) < m^*$. Thus $m^{**}\notin M$, $A(\x,x',G) = \I(\x,x',G)$. We then have,

\begin{align*}
    \frac{A_i(\x,x',F)}{\x_i-x_i'} &= \frac{A_n(\x,x',G)}{\x_n-x_n'}\\
    &= \frac{\I_n(\x,x',G)}{\x_n-x_n'}\\
    &= \frac{m_i^*}{\|m\|_0} (x_1-x_1')^{m_1^*}\dotsb (x_i-x_i')^{m_i^*-1} \dotsb (x_n-x_n')^{m_n^*}\\
    &= \frac{m_i^*}{\|m\|_0} \frac{(x_1-x_1')^{m_1^*}\dotsb (x_i-x_i')^{m_i^*} \dotsb (x_n-x_n')^{m_n^*}}{\x_i-x_i'}\\
    &= \frac{\I_i(\x,x',F)}{\x_i-x_i'}
\end{align*}

This shows that $A_i(\x,x',F) = \I_i(\x,x',F)$ for $i<n$. By completeness, we have $A_n(\x,x',F) = \I_n(\x,x',F)$. Thus $m^*\in M$, a contradiction. Thus there is no element of $\N_0^n$ that is not an element of $M$, and $M = \N_0^n$ concluding the proof.
\end{proof}

\noindent We now move to the main proof:
\vspace{2mm}

    Let $A\in \A^1$ satisfy completeness, dummy, linearity, and symmetric monotonicity and let $F\in \F^1$. For any $i$ such that $1\leq i \leq n$, $\frac{\partial F}{\partial x_i}$ is analytic and by the Stone Weierstrass theorem, for any $\epsilon > 0$, there exists a polynomial, $p$, such that $|p(x)-\pxi(x)| < \epsilon$ on $[a,b]$. Let $p_m$ be a polynomial such that $|p_m(x)-\pxi(x)| < \frac{1}{2m}$, and let $P_m$ be any polynomial so that $\frac{\partial P_m}{\partial x_i} = p_m$. Note that $\frac{\partial (P_m-\frac{x_i}{m})}{\partial x_i} = p_m - \frac{1}{m}< \pxi$.

    Now assume that $\x_i-x_i'\geq 0$. By symmetric monotonicity we have $A_i(\x,x',P_m-\frac{x_i}{m})\leq A_i(\x,x',F)$. Employing the dominated convergence theorem, we have:

    \begin{align*}
        A_i(\x,x',F) &\geq \lim_{m\rightarrow\infty} A_i(\x,x',P_m-\frac{x_i}{m})\\
        &= \lim_{m\rightarrow\infty} \I_i(\x,x',P_m-\frac{x_i}{m}) \\
        &= \lim_{m\rightarrow\infty} (\x_i-x_i') \int_0^1 \frac{\partial (P_m-\frac{x_i}{m})}{\partial x_i}(\g(t))dt\\
        &= \lim_{m\rightarrow\infty} (\x_i-x_i') \int_0^1 p_m(\g(t))dt - \frac{(\x_i-x_i')}{m}\\
        &= (\x_i-x_i') \int_0^1 \pxi(\g(t))dt\\
        &= \I_i(\x,x',F)
    \end{align*}
    By considering $P_m+\frac{x_i}{m}$, we gain the opposite inequality, namely, $A_i(\x,x',F)\leq \I_i(\x,x',F)$. This establishes that $A_i(\x,x',F) = \I_i(\x,x',F)$.

    The case where $\x_i-x_i'\leq 0$ follows a parallel proof.
\end{proof}

\section{Proof of Theorem~\ref{Theorem: symmetric monotonicity for A2}}\label{appendix: symmetric monotonicity for A2}

\begin{proof}
    (iii. $\Rightarrow$ ii.) Suppose $A\in \A^2(D_\IG)$ is the IG method and $(\x,x',F)\in D_\IG$. It is well known that  IG satisfies completeness, dummy, and linearity. If $\x_i = x_i'$, then it is easy to see that $\IG_i(\x,x',F) = 0$.
    
    Suppose that $(\x,x',G)\in D_\IG$ as well, and that $\x_i \neq x_i'$, $\x_j \neq x_j'$. Furthermore, suppose that $\pxi \leq \pxj$ locally approximately. Because $(\x,x',F)$, $(\x,x',F)\in D_\IG$, $\pxi$ and $ \pxj$ can be integrated along the path $\g(t)=x' +t(\x-x')$, implying that the measure of points on the path where $\pxi$ and $ \pxj$ exist has full measure with respect to the Lebesgue measure on $\R$. Suppose $x$ is one such point. Then $\lim_{z\rightarrow \infty}\frac{F(x_1,...,x_i+z,...,x_n) - F(x)}{z}$, $ \lim_{z\rightarrow \infty} \frac{G(x_1,...,x_j+z,...,x_n) - G(x)}{z}$ both exist and, because $\pxi \leq \pxj$ locally approximately, $\pxi(x) = \lim_{z\rightarrow \infty}\frac{F(x_1,...,x_i+z,...,x_n) - F(x)}{z} \leq \lim_{z\rightarrow \infty} \frac{G(x_1,...,x_j+z,...,x_n) - G(x)}{z} = \pxj(x)$. Thus,

    \begin{equation*}
        \begin{split}
            \frac{\IG_i(\x,x',F)}{\x_i-x_i'} &= \int_0^1 \pxi(x'+t(\x-x'))dt\\
            &\leq \int_0^1 \pxj(x'+t(\x-x'))dt\\
            &= \frac{\IG_j(\x,x',F)}{\x_j-x_j'}
        \end{split}
    \end{equation*}
    and IG satisfies $\mathcal{C}^0$-symmetric monotonicity.
    
    \vspace{2mm}

    ii. $\Rightarrow$ i.) Suppose $A\in \A^2(D_\IG)$ satisfies completeness, dummy, linearity, and $\mathcal{C}^0$-symmetric monotonicity. $A$ satisfies symmetric monotonicity for $F\in\F^1$ immediately by the definition of partial derivatives. Suppose that $F$ is non-decreasing from $x'$ to $\x$ and let $(\x,x',F)\in D_\IG$. If $\x_i = x_i'$, then $A_i(\x,x',F) = 0$. Suppose $\x_i> x_i'$. As previously observed, $\pxi$ exists almost everywhere on the straight path $\g(t)$. Setting $G\equiv 0$, then $0 = \frac{\partial G}{\partial x_i} \leq \pxi$ almost approximately since $F$ is non-decreasing from $x'$ to $\x$ and $\x_i> x_i'$. Thus $0 = \frac{A_i(\x,x',G)}{\x_i-x_i'} \leq \frac{A_i(\x,x',F)}{\x_i-x_i'}$, and $0 \leq A_i(\x,x',F)$. If instead we assume that $\x_i<x_i'$, then $0 = \frac{\partial G}{\partial x_i} \geq \pxi$, and $0 = \frac{\partial G}{\partial x_i} \leq -\pxi$. Thus $0 = \frac{A_i(\x,x',G)}{\x_i-x_i'} \leq \frac{A_i(\x,x',-F)}{\x_i-x_i'}$, and $0 \leq A_i(\x,x',F)$. Thus, in any case, $A_i(\x,x',F)\geq 0$ for all $i$, and $A$ satisfies NDP.

    \vspace{2mm}

    i. $\Rightarrow$ iii.) Suppose $A\in\A^2(D_\IG)$ satisfies completeness, dummy, linearity, and NDP. Let $(\x,x',F)\in D_\IG$ and choose a component $i$. By methods found in the proof of \cite[Theorem 2]{lundstrom2022rigorous}, there exists a sequence of functions $F_m$ such that:

     \begin{itemize}
         \item $F_m$ is analytic for all $m$.
         \item $\frac{\partial F_m}{\partial x_i} \leq \frac{\partial F}{\partial x_i}$ where $\frac{\partial F}{\partial x_i}$ exists.
         \item $\lim_{m\rightarrow \infty} \frac{\partial F_m}{\partial x_i} = \frac{\partial F}{\partial x_i}$ where $\frac{\partial F}{\partial x_i}$ exists.
         \item $|\frac{\partial F_m}{\partial x_i}|\leq k$ for all $m$.
         \item $F-F_m$ is non-decreasing from $x'$ to $\x$ in $i$.
     \end{itemize}
     
    By NDP we have $A_i(\x,x',F-F_m)\geq 0$ and $A_i(\x,x',F) \geq A_i(\x,x',F_m)$. Since $F_m\in \A^1$, we have $A_i(\x,x',F_m) = \IG_i(\x,x',F)$ by Theorem~\ref{Theorem: symmetric monotonicity for A1}. Recalling that $\pxi$ exists almost everywhere on IG's path, we employ the dominated convergence theorem to gain:

    \begin{align*}
        A_i(\x,x',F) &\geq \lim_{m\rightarrow\infty} A_i(\x,x',F_m)\\
        &= \lim_{m\rightarrow\infty} \I_i(\x,x',F_m) \\
        &= \lim_{m\rightarrow\infty} (\x_i-x_i') \int_0^1 \frac{\partial F_m}{\partial x_i}(\g(t))dt\\
        &= (\x_i-x_i') \int_0^1 \pxi(\g(t))dt\\
        &= \I_i(\x,x',F)
    \end{align*}
    By a parallel method we can gain $A_i(\x,x',F) \leq \IG_i(\x,x',F)$.

\end{proof}

\section{Softplus Approximations Converge Uniformly}\label{appendix: softplus approximation uniformly converges}

Define $S^k_\a$ to be as $S^k$, but replace each ReLU function $s$ in $S^k$ with the parameterized softplus, $s_\a$. Then the softplus approximation of $F$ is given by:

$$F_\a(x) = S_\a^m \circ F^m \circ S_\a^{m-1} \circ F^{m-1} \circ ... \circ S_\a^2 \circ F^2 \circ S_\a^1 \circ F^1 (x)$$

\begin{lemma}
$F_\a \rightarrow F$ uniformly on $U$.
\end{lemma}
\begin{proof}
Begin proof by induction. For $k=1$, it is easy to show that $s_\a\rightarrow s$ uniformly on $\R$, and thus, $S_\a^1\rightarrow S^1$ uniformly on $\R^n$. Thus, for any $\epsilon > 0$, an $A>0$ may be chosen such that for any $y\in \R^n$, $\a>A$ implies $\|S^1_\a(y) - S^1(y)\| < \epsilon$. Replace $y$ with $F^1(x)$ to get $S^1_\a (F^1) \rightarrow S^1(F^1)$ uniformly.

Write $ G^k := S^k \circ F^k \circ ... \circ S^1 \circ F^1 (x)$ and $G^k_\a := S_\a^k \circ F^k \circ ... \circ S_\a^1 \circ F^1 (x)$, and suppose $G_\a \rightarrow G$ uniformly. It remains to be shown that $S_\a^k \circ F^k \circ G^k_\a \rightarrow S^k \circ F^k \circ G^k$ uniformly.

\begin{align*}
    &\|S_\a^k ( F^k ( G^k_\a(x))) - S^k ( F^k ( G^k(x)))\|\\
    \leq &\|S_\a^k ( F^k ( G^k_\a(x))) - S_\a^k ( F^k ( G^k(x)))\| + \|S_\a^k ( F^k ( G^k(x))) - S^k ( F^k ( G^k(x)))\|\\
    \leq  & \|F^k ( G^k_\a(x)) - F^k ( G^k(x))\| + \|S_\a^k ( F^k ( G^k(x))) - S^k ( F^k ( G^k(x)))\|\\
\end{align*}
Where the third line is because $S^k_\a$ is Lipschitz with Lipschitz constant $\leq 1$.

Since $G^k$ is analytic, it is bounded on $U$. Since $G^k_\a$ converges uniformly to $G^k$, it is bounded for large enough $\a$. Let $\a_0$ produce this bound, that is, if $\a > \a_0$, then $\max(\|G^k_\a(x)\|, \|G^k(x)\|) \leq C_1$ for any $x\in U$. Since $F$ is analytic, it is Lipshitz on bounded domains. Thus, if $\a > \a_0$, then

$$\|F^k ( G^k_\a(x)) - F^k ( G^k(x))\| \leq  C_2\|G^k_\a(x) - G^k(x)\|$$

Now, by uniform continuity of $G_\a^k$ and $S_\a^k$, choose $\a_1$ so that $\a > \alpha_1$ guarantees that $\|G^k_\a(x) - G^k(x)\| < \epsilon /2C_2$, and choose $\a_2$ so that $\a > \alpha_2$ guarantees that $\|S_\a^k ( F^k ( G^k(x))) - S^k ( F^k ( G^k(x)))\| < \epsilon /2$. Then $\a>\max(\a_0,\a_1,\a_2)$ guarantees that
\begin{align*}
    &\|S_\a^k ( F^k ( G^k_\a(x))) - S^k ( F^k ( G^k(x)))\|\\
    \leq  & \|F^k ( G^k_\a(x)) - F^k ( G^k(x))\| + \|S_\a^k ( F^k ( G^k(x))) - S^k ( F^k ( G^k(x)))\|\\
    \leq &  C_2\|G^k_\a(x) - G^k(x)\| + \|S_\a^k ( F^k ( G^k(x))) - S^k ( F^k ( G^k(x)))\|\\
    < & \epsilon/2 + \epsilon/2 =\epsilon
\end{align*}
showing that $S_\a^k \circ F^k \circ G^k_\a \rightarrow S^k \circ F^k \circ G^k$ uniformly.
\end{proof}

\section{Proof of Theorem~\ref{theorem: softplus approximation topology}}\label{appendix: softplus approximation topology}
\subsection{Setup}

Define $S^k_\a$ to be as $S^k$, but replace each ReLU function $s$ in $S^k$ with the parameterized softplus, $s_\a$. Then the softplus approximation of $F$ is given by:

$$F_\a(x) = S_\a^m \circ F^m \circ S_\a^{m-1} \circ F^{m-1} \circ ... \circ S_\a^2 \circ F^2 \circ S_\a^1 \circ F^1 (x)$$

Also, for a funciton $G:\R^n\rightarrow \R^m$, define $DF$ to be the Jacobian, so that if $F_i$ is the $i^\text{th}$ output of $F$, then $(DG)_{i,j} = \frac{\partial G_i}{\partial x_j}$.

\subsection{Main Proof}
First, we state an outline of the proof. We proceed by induction. In the non-trivial case with one-dimensional output, $F^1$ is not the zero function and $S^1$ is ReLU. In this case, $\{y\in U: F^1(y)\neq 0\}$ is open and has full measure. For any $x$ in this set, we can compose $F^1$ with ReLU and get that $S^1\circ F^1$ behaves like $F^1$ or the zero function locally. For each $x$ in this set, $D (S^1_\a \circ F^1)$ converges locally to $DF^1$ or $0$ locally. In the multivariate case, each $(S\circ F^1)_i$ has a set with desired behaviors, so for any $x$ in the intersection of such sets, $S\circ F^1$ has the desired behaviors. That set is open and has full measure.

For the induction step, we assume that $G^k$ has the desired properties and want to show $S^{k+1}\circ F^{k+1}\circ G^k$ does as well. If $G^k$ is equivalent to an analytic function in some neighborhood, so is $F^{k+1}\circ G^k$. An argument similar to the $k=1$ step shows that for almost every $x$ in our neighborhood, $S^{k+1}\circ F^{k+1}\circ G^k$ is equivalent to an analytic function in some new open neighborhood containing $x$, and $S_\a^{k+1}\circ F^{k+1}\circ G_\a^k$ converges. We then consider a collection of points $x \in U$ with the desirable properties, and a collection of open sets $N_x$ containing them, where $S^{k+1}\circ F^{k+1}\circ G^k$ is locally equivalent to an analytic function on $N_x$. We show that $\cup_x N_x$ is open and has full measure.

\begin{proof}
Let $F\in\F^1$. As before, write $ G^k := S^k \circ F^k \circ ... \circ S^1 \circ F^1$ and $G^k_\a := S_\a^k \circ F^k \circ ... \circ S_\a^1 \circ F^1$. Assume that there exists $U^*\subset U$ with same measure as $U$, and that $x\in U^*$ implies that exists an open region containing $x$, $B_x$, such that: 1) $G^k\equiv H_x$ on $B_x$, where $H_x$ is a real-analytic function on $U$, 2) $DG^k(x)$ exists, and 3) $DG^k_\a(x)\rightarrow DG^k(x)$ as $\a\rightarrow \infty$. We want to show that there is a set analogous to $U^*$ for $S^{k+1} \circ F^{k+1} \circ G^k$ and $S_\a^{k+1} \circ F^{k+1} \circ G_\a^k$. With this established, we will have gained a proof by induction. To explain, the above is the $k\rightarrow k+1$ step. By setting $k=1$, and setting $F^1$, $S^1$ as the identity mappings, we will prove the $k=1$ step, concluding the proof.

First, let us consider the case where $F^{k+1}$, $S^{k+1}$ output in one dimension. Let $x\in U^*$, and suppose $G^k\equiv H_x$ on $B_x$. Then $F^{k+1} \circ G^k$ is analytic on $B_x$, since compositions of real analytic function are real analytic. 

\underline{Case 1}: Consider the case where $\lambda(\{y\in B_x:G^k(y) = 0)\}>0$. Then  $G^k\equiv 0$ and $S^{k+1} \circ F^{k+1} \circ G^k $ is constant on $B_x$. In this case, the derivative of $S^{k+1} \circ F^{k+1} \circ G^k$ exists everywhere on $B_x$, and is equal to zero. Now, for $y\in B_x$, we have 
\begin{align*}
    \lim_{\a\rightarrow \infty} \nabla (S_\a^{k+1} \circ F^{k+1} \circ G_\a^k)(y) & = \lim_{\a\rightarrow \infty} \sum_{j=1}^{n_k}\frac{d S_\a^{k+1}}{d (F^{k+1}\circ G^k)}(F^{k+1}(G^k_\a(y))) \times \frac{\partial F^{k+1}}{\partial G_{\a,j}^k}(G^k_\a(y)) \times \nabla G^k_{\a,j}(y)\\
    &= 0\\
    &= \nabla (S^{k+1} \circ F^{k+1} \circ G^k)(y)
\end{align*}
where the $0$ comes from the fact that $|\frac{d S_\a^{k+1}}{d (F^{k+1}\circ G^k)}| \leq 1$, $\frac{\partial F^{k+1}}{\partial G_{\a,j}^k}$ is bounded for a bounded domain (which it is), and $\nabla G^k_{\a,j}(y)\rightarrow 0$ for each $j$. Thus $S^{k+1} \circ F^{k+1} \circ G^k$, $S_\a^{k+1} \circ F^{k+1} \circ G_\a^k$ have properties 1-3 of the theorem on the set $B_x$.

\underline{Case 2}: Consider instead the case where $G^k$ is not the zero function, but $S^k$ is the identity mapping. Then $F^{k+1}\circ G^k$ is analytic on $B_x$ and so is $S^{k+1} \circ F^{k+1}\circ G^k$, and the derivative exists on $B_x$. Now, for $y\in B_x$, we have
\begin{equation}\label{F*G converge}
\begin{split}
    \lim_{\a\rightarrow \infty} \nabla (S_\a^{k+1} \circ F^{k+1} \circ G_\a^k)(y) & = \lim_{\a\rightarrow \infty} \nabla (F^{k+1} \circ G_\a^k)(y)\\
    &= \lim_{\a\rightarrow \infty} \sum_{j=1}^{n_k} \frac{\partial F^{k+1}}{\partial G_{\a,j}^k}(G^k_\a(y)) \times \nabla G^k_{\a,j}(y)\\
    &= \lim_{\a\rightarrow \infty} \sum_{j=1}^{n_k} \frac{\partial F^{k+1}}{\partial G_j^k}(G^k_\a(y)) \times \nabla G^k_{\a,j}(y)\\
    &= \sum_{j=1}^{n_k} \frac{\partial F^{k+1}}{\partial G_j^k}(G^k(y)) \times \nabla G^k_j(y)\\
    &= \nabla (F^{k+1} \circ G^k)(y)
\end{split}
\end{equation}
To explain the fourth line, $\nabla G^k_{\a,j}(y)$ converges pointwise by assumption. Also, $\frac{\partial F^{k+1}}{\partial G_j^k}$ is Lipschitz continuous in a bounded domain and $G^k_\a(y)$ converges uniformly. Thus each term converges pointwise. Thus $S^{k+1} \circ F^{k+1} \circ G^k$, $S_\a^{k+1} \circ F^{k+1} \circ G_\a^k$ have properties 1-3 of the theorem on the set $B_x$.

\underline{Case 3}: Consider the case where $G^k$ is not the zero function and $S^{k+1}$ is the ReLU function. Then $F^{k+1}\circ G^k$ is analytic and either the zero function or not on $B_x$.

\underline{Case 3.1}: Consider the subcase where $F^{k+1}\circ G^k\equiv 0$ on $B_x$. Then $S^{k+1}\circ F^{k+1}\circ G^k\equiv 0$ on $B_x$, is differentiable on $B_x$, and the derivative is the zero function. Then for $y\in B_x$, we have 
\begin{align*}
    \lim_{\a\rightarrow \infty} \nabla (S_\a^{k+1} \circ F^{k+1} \circ G_\a^k)(y) & = \lim_{\a\rightarrow \infty} \frac{d S_\a^{k+1}}{d (F^{k+1}\circ G_\a^k)}(F^{k+1}(G^k_\a(y))) \times \nabla (F^{k+1} \circ G^k_\a)(y)\\
    &= \lim_{\a\rightarrow \infty} \frac{d S_\a^{k+1}}{d (F^{k+1}\circ G_\a^k)}(F^{k+1}(G^k_\a(y))) \times \nabla (F^{k+1} \circ G^k)(y)\\
    &= \lim_{\a\rightarrow \infty} \frac{d S_\a^{k+1}}{d (F^{k+1}\circ G^k)}(F^{k+1}(G^k_\a(y))) \times 0\\
    &= \nabla (S^{k+1} \circ F^{k+1} \circ G^k)(y)
\end{align*}
where the third line is because $\frac{d S_\a^{k+1}}{d (F^{k+1}\circ G^k)}$ is bounded and $\nabla (F^{k+1} \circ G^k_\a)\rightarrow \nabla (F^{k+1} \circ G^k)$ on $B_x$ by Eq.~\eqref{F*G converge}. Thus in this subcase, $S^{k+1} \circ F^{k+1} \circ G^k$, $S_\a^{k+1} \circ F^{k+1} \circ G_\a^k$ have properties 1-3 of the theorem on the set $B_x$.

\underline{Case 3.2}: Instead consider the subcase where $F^{k+1}\circ G^k$ is a non-constant function on $B_x$. We have $\lambda(\{z\in B_x:F^{k+1}\circ G^k(z) = 0\}) = 0$.

\underline{Case 3.2.1}: Suppose $F^{k+1}\circ G^k(x)>0$. Because $F^{k+1}\circ G^K$ is continuous, there exists an open set $B_x'$ containing $x$ where $F^{k+1}\circ G^k>0$, and that on such a set, $S^{k+1}\circ F^{k+1}\circ G^k \equiv  F^{k+1}\circ G^k$. Then, 
\begin{align*}
    \lim_{\a\rightarrow \infty} \nabla (S_\a^{k+1} \circ F^{k+1} \circ G_\a^k)(y) & = \lim_{\a\rightarrow \infty} \frac{d S_\a^{k+1}}{d (F^{k+1}\circ G^k)}(F^{k+1}(G^k_\a(y))) \times \nabla (F^{k+1} \circ G^k_\a)(y)\\
    &= 1 \times \nabla (F^{k+1} \circ G^k)(y)\\
    &= \nabla (S^{k+1} \circ F^{k+1} \circ G^k)(y)
\end{align*}

\underline{Case 3.2.2}: Suppose $F^{k+1}\circ G^k(x)<0$. Because $F^{k+1}\circ G^K$ is continuous, there exists an open set $B_x'$ containing $x$ where $F^{k+1}\circ G^k<0$, and that on such a set,, $S^{k+1}\circ F^{k+1}\circ G^k \equiv  0$. Then, 
\begin{align*}
    \lim_{\a\rightarrow \infty} \nabla (S_\a^{k+1} \circ F^{k+1} \circ G_\a^k)(y) & = \lim_{\a\rightarrow \infty} \frac{d S_\a^{k+1}}{d (F^{k+1}\circ G^k)}(F^{k+1}(G^k_\a(y))) \times \nabla (F^{k+1} \circ G^k_\a)(y)\\
    &= 0 \times \nabla (F^{k+1} \circ G^k)(y)\\
    &= \nabla (S^{k+1} \circ F^{k+1} \circ G^k)(y)
\end{align*}
\underline{Case 3.2.3}: Suppose $F^{k+1}\circ G^k(x)=0$. In this case, we do not define a $B_x'$ set. We remind the reader that if $x\in U^*$ is a case 3.2.3 point, then $\lambda(\{z\in B_x:F^{k+1}\circ G^k(z) = 0\}) = 0$

Thus we have established in the one-dimensional output case that for each $x\in U^*$ that is not a case 3.2.3 point, there exists an open neighborhood containing $x$ where properties 1-3 hold.

Now consider the multivariate case. Define $K\subset U^*$ as the set of points in $U^*$ that are case 3.2.3 points for at least one output of $S^{k+1}\circ F^{k+1}\circ G^k$. Let $x\in U^*\setminus K$. Let $B_{x,i}'$ correspond to the open set containing $x$ where properties 1-3 hold when we only consider the output $(S^{k+1} \circ F^{k+1}\circ G^k)_i$. Then properties 1-3 hold on $\cap_i B_{x,i}'$ for each output of $(S^{k+1} \circ F^{k+1}\circ G^k)_i$ and $(S_\a^{k+1} \circ F^{k+1}\circ G_\a^k)_i$. Thus properties 1-3 hold for $S^{k+1} \circ F^{k+1}\circ G^k$ and $S_\a^{k+1} \circ F^{k+1}\circ G_\a^k$ on $\cap_i B_{x,i}^*$. Thus we have established in the multivariate case the following: for each $x\in U^*\setminus K$, there exists an open neighborhood containing $x$, $B_x'$, where properties 1-3 hold for $S^{k+1}\circ F^{k+1}\circ G^k$ and $S_\a^{k+1}\circ F^{k+1}\circ G_\a^k$.

We now move to show that $\lambda(K) = 0$, which will conclude the proof. Let $K_i$ denote the set of case 3.2.3 points for the $i^\text{th}$ output of $S^{k+1}\circ F^{k+1}\circ G^k$. Since $K=\cup_i K_i$, it suffices to show $\lambda(K_i)=0$. Let $x\in K_i$ for some $i$. Then $x$ is a case 3.2.3 point for the output of $(S^{k+1}\circ F^{k+1}\circ G^k)_i$. According to our assumption, there exists a $B_x$ containing $x$ where properties 1-3 hold for $G^k$, $G^k_\a$ on $B_x$. Note that $\cup_{x\in K_i} B_x$ is an open cover, and has a countable subcover $\cup_{j\in \N} B_{x_j}$, where each $x_j$ is a case 3.2.3 point for $(S^{k+1}F^{k+1}\circ G^k)_i$. Because $K_i\subseteq \cup_{x\in K_i} B_x$, we also have $K_i\subseteq \cup_{j\in \N} B_{x_j}$. Now,

\begin{align*}
    K_i &= K_i\cap (\cup_{j\in \N} B_{x_j})\\
    &= \cup_{j\in \N} (B_{x_j}\cap K_i)
\end{align*}
Now, if $x\in K_i$, then $(F^{k+1}\circ G^k)_i(x)=0$ by virtue of being a case 3.2.3 point. Also, it has been established that for case 3.2.3 points, $\lambda(\{z\in B_x:(F^{k+1}\circ G^k)_i(z) = 0\}) = 0$. Thus $\lambda(B_{x_j}\cap K_i)\leq \lambda(\{z\in B_{x_j}:(F^{k+1}\circ G^k)_i(z) = 0\}) = 0$. Thus, $K_i$ is a countable union of sets of measure zero, and is thus measure zero.
\end{proof}

\section{Proof of Corollary~\ref{corollary: convergence of softplus approximaiton}}\label{appendix: convergence of softplus approximaiton}
\begin{proof}
Let $F\in \F^2$, and let $U$ be the set as in Theorem~\ref{theorem: softplus approximation topology}. Let $\g(t)$ be the uniform speed path from $x'$ to $\x$ and suppose $\lambda(\{t\in [0,1]: \g(t)\in U\}) = 1$, where $m$ is the Lebesgue measure on $\R$. By Theorem~\ref{theorem: softplus approximation topology}, we have $\nabla F_\a(\g(t)) \rightarrow \nabla F(\g(t))$ for almost every $t$ in $[0,1]$. Suppose $\nabla F_\a$ is bounded on $U$ for large enough $\a$. Let $a_n$ be any sequence such that $a_n\rightarrow \infty$. Choose any index $i$, and by employing the dominated convergence theorem we gain:
\begin{align*}
    \lim_{n\rightarrow \infty} \I_i(\x,x',F_{a_n}) &= \lim_{n\rightarrow \infty} (\x_i-x_i')\int_0^1\frac{\partial F_{a_n}}{\partial x_i}(\g(t)) dt\\
    &= (\x_i-x_i')\int_0^1\frac{\partial F}{\partial x_i}(\g(t)) dt\\
    &= \I_i(\x,x',F)
\end{align*}
Since $\lim_{n\rightarrow \infty} \I_i(\x,x',F_{a_n}) = \I_i(\x,x',F)$ for any sequence $a_n$, we have $\lim_{\a\rightarrow \infty} \IG_i(\x,x',F_\a) = \IG_i(\x,x',F)$.

We now turn to show that $\nabla F_\a$ is bounded for large enough $\a$. Using the notation introduced in Theorem~\ref{theorem: softplus approximation topology}, note that:
$$\nabla F_\a = DS_\a^m \, DF^m\, DS_\a^{m-1}\, DF^{m-1}  ... DS_\a^2 \,DF^2 \, DS_\a^1 \, DF^1$$
Thus,
\begin{align*}
\|\nabla F_\a\|_\infty &\leq \Pi_{k=1}^m \|DS_\a^{k}(F^k \circ...\circ F^1)\|_\infty \times \|DF^{k}(S_\a^{k-1} \circ...\circ F^1)\|_\infty
\end{align*}
Now $\|DS_\a^{k}(F^k \circ...\circ F^1)\|_\infty\leq 1$ since $S_\a^k$ is either softplus or the identity mapping for each input. Also, $F^k$ is Lipshitz in a bounded domain, and $S_\a^{k-1} \circ...\circ F^1$ converges uniformly on $U$ to a function with a bounded range. Thus $\|DF^{k}(S_\a^{k-1} \circ...\circ F^1)\|_\infty$ is bounded on $U$, and $\|\nabla F_\a\|_\infty$ is bounded.
\end{proof}

\end{document}